\documentclass[letterpaper]{article} 
\usepackage[]{aaai2026}  
\usepackage{times}  
\usepackage{helvet}  
\usepackage{courier}  
\usepackage[hyphens]{url}  
\usepackage{graphicx} 
\usepackage{natbib}  
\usepackage{caption} 
\usepackage{algorithm}
\usepackage{algorithmic}

\usepackage{newfloat}
\usepackage{listings}
\DeclareCaptionStyle{ruled}{labelfont=normalfont,labelsep=colon,strut=off} 
\floatstyle{ruled}
\newfloat{listing}{tb}{lst}{}
\floatname{listing}{Listing}
\title{Where Common Knowledge Cannot Be Formed, Common Belief Can -- Planning with Multi-Agent Belief Using Group Justified Perspectives (with Supplementary Material)}
\author{
    Guang Hu\textsuperscript{\rm 1},
    Tim Miller\textsuperscript{\rm 2},
    Nir Lipovetzky\textsuperscript{\rm 1},\\
}
\affiliations{
    \textsuperscript{\rm 1}The University of Melbourne\\
    \textsuperscript{\rm 2}The University of Queensland

    ghu1@student.unimelb.edu.au, timothy.miller@uq.edu.au, nir.lipovetzky@unimelb.edu.au
}

\usepackage{bibentry}

\usepackage{amsthm}
\usepackage{multirow}
\usepackage{multicol}
\usepackage{booktabs}
\usepackage{amsmath}
\usepackage{verbatim}
\usepackage{MnSymbol}
\usepackage{mathtools}
\usepackage{amsfonts}
\usepackage[hang,flushmargin]{footmisc}

\usepackage{amsmath}
\usepackage{supertabular}

\usepackage{enumitem}
\setlist[enumerate]{after=\vspace{-0.5\baselineskip}}

\usepackage{caption}

\usepackage{amsmath}
\usepackage{amsthm}
\newtheoremstyle{normalfont}
  {3pt} 
  {3pt} 
  {\normalfont} 
  {} 
  {\bfseries} 
  {.} 
  { } 
  {} 

\theoremstyle{normalfont}
\newtheorem{example}{Example}

\usepackage{booktabs}
\newtheorem{definition}{Definition}

\newtheorem{theorem}{Theorem}
\newtheorem{plan}{Plan}

\usepackage{listings}

\def\L{\mathcal{L}}
\def\T{T}
\def\unknown{\frac{1}{2}}

\usepackage[disable]{todonotes}

\newcommand{\tm}[1]{\todo[author=Tim,inline,color=green]{#1}}

\usetikzlibrary{calc,fadings,decorations.pathreplacing,automata}

\usepackage[strict]{changepage}

\usepackage{framed}

\definecolor{formalshade}{rgb}{0.95,0.95,1}
\definecolor{darkblue}{rgb}{0.0,0.0,0.5} 

\usepackage[many]{tcolorbox}
\definecolor{formalshade}{rgb}{0.95,0.95,1}
\definecolor{darkblue}{rgb}{0.0,0.0,0.5}

\newtcolorbox{myquote}[1][]{%
  colback=formalshade, colframe=darkblue,
  left=4pt, right=4pt, top=4pt, bottom=4pt,
  boxrule=0.5pt,
  sharp corners,
  width=\columnwidth,
  enhanced,
  breakable,
  #1
}

\newcommand{\defmathend}{\tag*{\text{$\blacksquare$}}}
\newcommand{\defnormalend}{\hfill$\blacksquare$}

\newcommand{\relationbasecase}{r(V_r)}

\newcommand{\relationset}{\mathcal{R}}

\newcommand{\pwpdomains}[0]{ \mathbb{D}}
\newcommand{\signature}{\Sigma \!\!=\!\! (Agt, V, \pwpdomains, \relationset)}

\newcommand{\jpmodel}{M=(Agt, V,\pwpdomains,\pi,\observation_1,\dots,\observation_k)}

\newcommand{\override}[1]{\langle #1 \rangle}

\newcommand{\dom}{\textrm{dom}}

\newcommand{\timestamp}{ts}
\newcommand{\seq}{\Vec{s}}
\newcommand{\bool}{\{true,false\}}

\newcommand{\statespace}{\mathcal{S}}
\newcommand{\seqspace}{\vec{\mathcal{S}}}
\newcommand{\statespacecomplete}{\mathcal{S}_{c}}

\newcommand{\f}{\mathit{f}}
\newcommand{\observation}{\mathit{O}}

\newcommand{\oldobservation}{\mathit{O}}
\newcommand{\memorization}{\mathit{R}}

\newcommand{\none}{\!\perp}
\newcommand{\assign}{\!=\!}

\newcommand{\cc}{\mathit{c}}

\newcommand{\es}{\mathit{ES}}
\newcommand{\ds}{\mathit{DS}}
\newcommand{\cs}{\mathit{CS}}
\newcommand{\ek}{\mathit{EK}}
\newcommand{\dk}{\mathit{DK}}
\newcommand{\ck}{\mathit{CK}}
\newcommand{\eb}{\mathit{EB}}
\newcommand{\db}{\mathit{DB}}
\newcommand{\cb}{\mathit{CB}}
\newcommand{\ef}{\mathit{ef}\!}
\newcommand{\df}{\mathit{df}\!}
\newcommand{\cf}{\mathit{cf}\!}

\begin{document}

\maketitle

\begin{abstract}
Epistemic planning is the sub-field of AI planning that focuses on changing knowledge and belief. 
It is important in multi-agent domains where agents need to have knowledge/belief regarding the environment, but also the beliefs of other agents, including nested beliefs.
When modeling knowledge in multi-agent settings, many models face an exponential growth challenge in terms of nested depth.
A contemporary method, known as \emph{Planning with Perspectives}~(PWP), addresses these challenges through the use of perspectives and set operations for knowledge.
Furthermore, the \emph{Justified Perspective}~(JP) model defines that an agent's belief is \emph{justified} if and only if the agent has seen evidence that this belief was true in the past and has not seen evidence to suggest that this has changed.

The current paper extends the JP model to handle \emph{group belief}, including distributed belief and common belief, even mixing knowledge and belief modalities.
We call this the \emph{Group Justified Perspective} (GJP) model.
Using experimental problems crafted by adapting well-known benchmarks to a group setting, we show the efficiency and expressiveness of our GJP model. 
The end result is the only planning tool capable of solving problems involving distributed belief and common belief.

\end{abstract}

\section{Introduction and Motivation}
\label{sec:intro}
Epistemic planning is a sophisticated branch of automated planning that integrates elements from both classical planning and epistemic logic. 
It allows the agents to reason about not only the physical world but also other agents' knowledge and beliefs.
It is suitable for solving multi-agent cooperative or adversarial tasks.

There are two traditional research directions to solving epistemic planning problems:
explicitly maintain all epistemic relations, such as Kripke frames~\cite{DBLP:conf/aips/KominisG15,DBLP:journals/jancl/BolanderA11,DBLP:conf/ecsi/Bolander14};
or require an expensive pre-compilation step to convert an epistemic planning problem into a classical planning problem~\cite{DBLP:journals/ai/MuiseBFMMPS22,DBLP:journals/ai/CooperHMMPR21}.
Research from both directions faces exponential growth in terms of the epistemic formulae depth.

Recently, \citeauthor{DBLP:journals/jair/Hu0L22}~\shortcite{DBLP:journals/jair/Hu0L22} proposed a lazy state-based approach called \emph{Planning with Perspectives} (PWP) that uses F-STRIPS~\cite{geffner2000functional} to reason about the agent's seeing relation and knowledge.
Their intuition is to use \textbf{perspective functions} to model the part of a state that each agent can see, and evaluate epistemic formulae from this. 
In short, an agent knows a proposition if it can see the variables involved in the proposition, and that proposition is true. 
They allow perspective functions to be implemented in F-STRIPS external functions, which means new logics can be created; for example, they model proper epistemic knowledge bases~\cite{DBLP:conf/ecai/LakemeyerL12} and Big Brother logic~\cite{DBLP:conf/atal/GasquetGS14} in continuous domains, with impressive computational results.
\citeauthor{DBLP:conf/aips/Hu0L23}~\shortcite{DBLP:conf/aips/Hu0L23} extended their model to model belief as well as knowledge, permitting e.g.\ conflicting belief between agents.
However, their model could only reason about single agent nested belief, not group belief operators such as common belief. 

In this paper, we extend their work to model uniform belief, distributed belief, and common belief.
We follow the intuition that when people reason about something they cannot see, they generate justified beliefs by retrieving the information they have seen in the past~\cite{goldman1979justified}.

However, applying this intuition of ``belief is past knowledge'' na\"ively to group belief is neither complete nor consistent. 
It is possible to form a common belief about a proposition even if there was no prior common knowledge about this. 
For example, consider agent $a$ looking in a box and seeing a coin with heads, and then agent $b$ looking into the box a minute after agent $a$ and seeing it is heads. 
At no point did they see the coin at the same time, so they cannot form common knowledge that the coin is heads (it may have changed in the minute in between). 
However, they can form a common belief that it is heads because they each saw heads and have no evidence to suggest the value has changed.

We illustrate this idea by extending the false-belief example from \cite{DBLP:conf/aips/Hu0L23}.
\begin{example}
\label{example:number}
    There are two agents $a$ and $b$, and there is a number $n \in \mathbb{N}$ inside a box. 
    The number can only be seen by the agents when they are peeking into the box. 
    The agents know whether the others are peeking into the box.
    The actions that agents can do are: \emph{peek} and \emph{return}.
    They cannot peek into the box at the same time, so they need to return to allow the other agent to peek.
    There are two hidden actions \emph{add} and \emph{subtract} (performed by another hidden agent), and their effects are only visible to the agents who are peeking into the box.
    Initially, both agents $a$ and $b$ are not peeking, and the value of $n$ is $2$.
    The task is to generate a plan such that:
     the common belief between $a$ and $b$ is $n\!\!<\!\!3$.
\end{example}

\begin{figure}[t!h]
    \centering
    \includegraphics[width=0.49\textwidth]{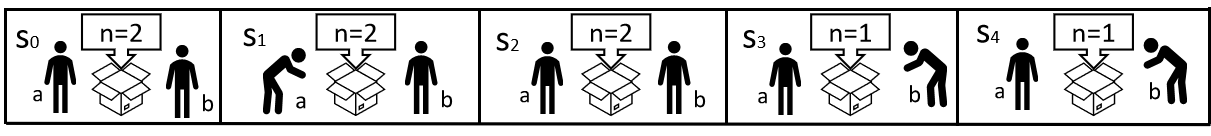}
    \caption{Plan~\ref{plan:1}.}
    \label{fig:plan_1}
\end{figure}

As shown in Figure~\ref{fig:plan_1}, a valid plan to achieve the common group belief above would be:
\begin{plan}
\label{plan:1}
    $peek(a)$, $return(a)$, $subtract$, $peek(b)$ 
\end{plan}

At the end of the plan, agent $a$ does not know if $n \!\assign\! 2$ -- an agent only knows something if they can see it.
\citeauthor{DBLP:conf/aips/Hu0L23}~\shortcite{DBLP:conf/aips/Hu0L23} model individual belief using \emph{justified belief} -- justified belief can be derived from an agent's ``memory''~\cite{goldman1979justified} -- an agent believes something if they knew it earlier and have no evidence that it has changed.

In the above, agents $a$ and $b$ do not peek into the box at the same time. 
So, at no point neither the statement ``agent $a$ knows that agent $b$ knows $n\!\!<\!\!3$'' ($K_a K_b n\!\!<\!\!3$) nor $K_b K_a n\!\!<\!\!3$ holds.
Further, the common knowledge $\ck_{\{a,b\}} n\!\!<\!\!3$ does not hold. 
However, we assert that the common belief $\cb_{\{a,b\}} n\!\!<\!\!3$ should hold if agents have memory. 
Since agent $a$ sees $n \!\assign\! 2$ after step 1 and agent $b$ sees $n \!\assign\! 1$ after step 4, both $B_a n \!\assign\! 2$ and $B_b n \!\assign\! 1$ hold, which implies both $B_a n\!\!<\!\!3$ and $B_b n\!\!<\!\!3$.
In addition, since agent $a$ sees agent $b$ peeking into the box after step 4 and $B_a n \!\assign\! 2$, $B_a B_b n \!\assign\! 2$ should hold.
Similarly, $B_b B_a n \!\assign\! 1$ should hold.
Therefore, we have both $B_a B_b n\!\!<\!\!3$ and $B_b B_a n\!\!<\!\!3$. 
Given that $a$ and $b$ both saw that each other peeked in the box, and saw that each saw that each peeked into the box, etc, both $a$ and $b$ believe each other believes $n\!\!<\!\!3$ with infinite depth.
From the definition by \citeauthor{Fagin:2003:RK:995831}~\shortcite{Fagin:2003:RK:995831}, this constitutes common belief.

In this paper, we propose group perspective functions to reason about uniform belief, distributed belief, and common belief, even mixed with individual and group knowledge operators.
Then, we define the semantics for these newly introduced group belief operators and show the axiomatic system they follow.
We discuss an implementation that extends an existing epistemic planning tool, and report experiments on key domains in epistemic planning.
Our results show that we can efficiently\footnote{Note: we do not have any existing approach to compare to.} and expressively solve interesting problems with group belief, even with a basic blind search algorithm, making this the only planner capable of solving problems of distributed and common belief.
\section{Background}
\label{sec:background}

\subsection{Epistemic planning}
\label{sec:background:epistemic_planning}
In epistemic planning, the most popular approach, Kripke structures~\cite{Fagin:2003:RK:995831}, model belief and knowledge logic using the \emph{possible worlds}.
The core idea is to use a set  of \emph{accessibility relations} $\mathcal{K}_i$ to represent whether agent $i$ can distinguish between two states (possible worlds). An agent knows (believes) a formula $\phi$ if it is true in every world that the agent considers possible.
The difference between knowledge and belief lies in the properties of $\mathcal{K}_i$.
For knowledge, the agent's accessibility relations need to be \emph{Reflexive} and \emph{Transitive} and \emph{Euclidean}, while the Reflexivity (Axiom T) is lost in belief.
Semantically speaking, if $K_i \varphi$ (agent $i$ knows $\varphi$ is true) holds (Axiom T), then $\varphi$ holds; while if $B_i \varphi$ (agent $i$ believes $\varphi$), it is not necessarily the case that $\varphi$ holds. In short: agents can have incorrect beliefs, but not incorrect knowledge. 
For group beliefs, there are mainly three types: uniform beliefs (shared beliefs); distributed beliefs; and common beliefs.

\textbf{Uniform belief}, denoted $\eb_G \varphi$, is straightforward --- it means that everyone in  group $G$ believes proposition $\varphi$. 
There are a number of approaches to model uniform belief~\cite{DBLP:phd/ethos/Iliev13,DBLP:journals/ai/FrenchHIK13}.


\textbf{Distributed belief}, denoted $\db_G \varphi$, combines the beliefs of all agents in group $G$.  
It is, effectively, the pooled beliefs of group $G$ if the agents were to ``communicate'' everything they believe to each other.
Any model has to consider the pooled beliefs from each agent and the pooled beliefs from the group that are not held by any of its individual agents, but are held by the group. 
For example, if agent $a$ believes $x=1$ (and nothing else) and agent $b$ believes $y=1$ (and nothing else), distributively, the group $\{a,b\}$ believes that $x=y$, even though no individual agent believes this.
Distributed belief is challenging because agents can have conflicting beliefs: if agent $a$ believes $x=1$ and agent $b$ believes $x=2$, what should the distributed belief be?

There are two main approaches to model distributed belief: (1) belief merging~\cite{DBLP:conf/kr/Konieczny00,DBLP:journals/logcom/KoniecznyP02,DBLP:conf/atal/EveraereKM15}; 
and (2) merging the agents' epistemic accessibility relations~\cite{DBLP:journals/ai/HalpernM92,DBLP:journals/synthese/WangA13,DBLP:journals/jancl/Roelofsen07,AGOTNES20171,DBLP:conf/dali/Solaki20,DBLP:conf/lori/Li21,DBLP:conf/lics/GoubaultKLR23,VANDERHOEK1999215}. 
Typically, merging conflicting beliefs is solved using some form of ordering over agents or propositions, meaning that some agents (propositions) receive priority over others. 
In this paper, we propose a definition for group justified distributed belief, in which conflicting beliefs are removed entirely, leading to a modal operator that obeys the axiom of consistency (axiom $D$) while dropping Axiom D2 (A group distributedly believes $\varphi$ if its subset distributedly believes $\varphi$) as agents' justified beliefs could be false.
In what is the most closely related work to ours, \citeauthor{DBLP:conf/ecai/HerzigLPRS20}~\shortcite{DBLP:conf/ecai/HerzigLPRS20} combine the two approaches of belief merging and the merging of accessibility relations to define a logic for modeling explicit and implicit distributed beliefs.
Explicit distributed belief is obtained from each agent's individual belief base; while implicit belief is derived from the group's collective belief base. In addition, they also introduce customized belief combination operators to model consistent distributed beliefs. 

\textbf{Common belief}, denoted $\cb_G \varphi$, is defined as: all agents in $G$ believe $\varphi$, all agents in $G$ believe that all agents in $G$ believe $\varphi$, all agents in $G$ believe \ldots, up to an infinite depth of nesting.
The existing work~\cite{Meggle2003,DBLP:conf/lori/Schwarzentruber11,DBLP:journals/mlq/Bonanno96,DBLP:journals/jphil/Heifetz99,DBLP:journals/mlq/BonannoN00} reasons for belief on belief bases or possible worlds.
\citeauthor{DBLP:journals/jair/Hu0L22}~\shortcite{DBLP:journals/jair/Hu0L22} forms the common knowledge of a group by finding the fixed point intersection of all agents' perspectives from the group, showing that this fixed point always exists within a finite bound.
However, their approach cannot handle beliefs.

\subsection{Planning with Perspectives}
\label{sec:background:pwp}

\citeauthor{DBLP:journals/jair/Hu0L22}~\shortcite{DBLP:journals/jair/Hu0L22} proposes a perspective model to lazily evaluate epistemic (knowledge) formulae with external functions, called \emph{Planning with Perspectives} (PWP).
They adapt the seeing operator $S_i$ from \citeauthor{DBLP:conf/ecai/CooperHMMR16}~\shortcite{DBLP:conf/ecai/CooperHMMR16} for individual agent $i$ following the intuition that ``seeing is knowing''; formally, $K_i\varphi \leftrightarrow \varphi \land S_i \varphi$. 
That is, agent $i$ knows $\varphi$ iff $\varphi$ is true and they see the value of $\varphi$.

Firstly, they defined the signature of their PWP model as:
\begin{definition}[PWP Signature]
\label{def:pwp:signature}
A \emph{signature} $\Sigma$ is described by the tuple $\signature$, with $Agt$ being a finite set of agent identifiers (of size $k$), $V$ as a finite set of variables (of size $m$) such that $Agt \subseteq V$ ($k \leq m$), implying agent identifiers serve as variables. 
Furthermore, $\pwpdomains$ denotes the set of all domains, where each $D_v$ corresponds to a possibly infinite domain of constant symbols (including a $\none$ symbol to represent ``None'') for each variable $v \in V$. 
Domains can be discrete or continuous. 
Lastly, $\relationset$ denotes a finite collection of predicate symbols.\defnormalend
\end{definition}

From this, they defined a state as a set of variable assignments $v \assign e$, where $v \in V$ and $e \in D_v$, such that each variable $v$ appears at most once in any state. 
They define $v \in s$ as syntactic sugar for $ \exists e \!\!\in\!\! D_v, v \!\assign\! e \in s$ and $s_{\none}$ as a special state ($s_{\none} \!\assign\! \{v \!\!= \none \mid v \!\!\in\!\! V\}$).
In addition, they denote $\statespace$ as the set of all possible states and $\statespacecomplete$ as the set of complete states ($\forall s_c \in \statespacecomplete, |s_c| \assign |V|$).
They also define a state override function $\override{}$ to update the state as follows:
\begin{definition}[State Override Function]
\label{def:state_override}
    A state override function $\langle \rangle: \statespace \times \statespace \to \statespace$ for a given state $s$ overriding a state $s'$ is defined as:
    \[
     s'\langle s \rangle = s \cup \{ v \assign s'(v) \mid v \in s' \land v \notin s \} \defmathend
    \] 
\end{definition}
That is, merging the two states but prioritizing the value in $s$ if there is a clash. 

The key idea in their approach is to reason about agents' epistemic formulae based on what the agents can observe from their local state.
This is done by defining a \textbf{perspective (observation) function} for each agent $i$ that takes a state and returns a subset of that state, which represents the part of the state that is observable to agent~$i$.

\begin{definition}[Observation Function]
\label{def:observation}
An observation function for agent~$i$, $\observation_i: \statespace \rightarrow \statespace$, is a function that takes a state and returns a subset of that state, representing the part of the state visible to agent~$i$. 
The following properties must hold for an observation function $\observation_i$ for all $i \in Agt$ and $s \in S$:
\begin{equation*}
    \begin{aligned}
            1. &\ \observation_i(s) \subseteq s , & \text{(Contraction)}\\
            2. &\ \observation_i(s) = \observation_i(\observation_i(s)), &  \text{(Idempotence)}\\
            3. &\ \text{ If } s \subseteq s', \text{ then } \observation_i(s) \subseteq \observation_i(s'), &  \text{(Monotonicity)}\\
    \end{aligned}\defmathend
\end{equation*} 
\end{definition}

These properties ensure that the observation function $\oldobservation_i$ is contractive, idempotent, and monotonic.

In addition, they also introduce group observation functions to model group knowledge.
They use the set union operator for distributed knowledge, and a fix point function for common knowledge.
They define a semantics—referred to as the \emph{complete semantics}—for both individual and group modal operators, based on the observation functions, and show that it is sound and complete with respect to the \textbf{KD45} axioms.
To enable more efficient evaluation, they also introduce a \emph{ternary semantics}, and prove that it agrees with the complete semantics on a fragment of the logic known as \emph{logically separable formulae}, which excludes tautologies and contradictions.

They provide some general perspective functions and show how it can be customized for specific domains. 
This provides a level of expressiveness not possible in declarative planning languages by implementing the semantics into external functions and using the encoding following F-STRIPS.
Over several benchmarks, PWP solves problems faster than the state-of-the-art approach~\cite{DBLP:journals/ai/MuiseBFMMPS22}.

\subsection{Justified Perspective Model}
\label{sec:background:jpm}
\citeauthor{DBLP:conf/aips/Hu0L23}~\shortcite{DBLP:conf/aips/Hu0L23} 
extend the PWP approach with the Justified Perspective (JP) model, which reasons about belief as well as knowledge.
They introduce the belief operator $B_i$ and reason about belief by generating the agents' \emph{justified perspectives} (following \citeauthor{goldman1979justified}~\shortcite{goldman1979justified}'s intuition), which are state sequences that agents believe in.


They inherited the signature $\Sigma$ (Definition~\ref{def:pwp:signature}) from the PWP approach, and defined the grammar of Knowledge-Belief~($KB$) language $\L_{KB}(\Sigma)$ as:
\[
\varphi ::= \relationbasecase \mid \neg \varphi \mid \varphi \land \varphi \mid S_i v \mid S_i \varphi \mid K_i \varphi \mid B_i \varphi,
\]
\noindent where $r \in \relationset$, $V_r \subseteq V$ are the terms of $r$ and $r(V_r)$ form a predicate (the set of all predicates denoted as $\mathcal{P}$), $v\in V$, and $i \in Agt$.
$S_i v$ and $S_i \varphi$ mean agent $i$ sees variable $v$ and sees formula $\varphi$ respectively, while $K_i \varphi$ and $B_i \varphi$ represent agent $i$ knows/believes formula $\varphi$.
The difference between knowledge and belief is that belief could be false (not consistent with the actual world).
Then, they defined their model as:
\begin{definition}[JP Model]
\label{def:jp:model}
    Given a signature $\Sigma$, a JP model instance can be represented by a tuple: 
    \[
        \jpmodel \defmathend
    \]
\end{definition}
They also followed the same definition of the states from the PWP approach above, as well as the perspective function $\oldobservation_i$.
The $Agt$, $V$, and $\pwpdomains$ are from the given signature.
The interpretation function $\pi: \statespace \times \mathcal{P} \to \bool$ evaluates whether the predicate $\relationbasecase$ is true in a given state ($\pi$ is undefined if there exists $v \!\!\in\!\! V_r$ such that $v \!\!\notin\!\! s$ or $s(v) \!= \none$). 
In addition, they denote $\seq$ as a sequence of states, $n$ as the last timestamp of any sequence, and $\seq[t]$ as the state at timestamp $t$.
Besides, they extend the state override function (Definition~\ref{def:state_override}) and observation function (Definition~\ref{def:observation}) to take input of a sequence.
Specifically, the extended forms are defined as $s' \override{\seq}=[s' \override{\seq[0]}, \dots, s' \override{\seq[n]}]$ and $\observation_i(\seq)=[\observation_i(\seq[0]),\dots,\observation_i(\seq[n])]$.

Then, they formed agents' justified perspectives by: the retrieval function and the justified perspective function.

The \textbf{retrieval function $\memorization$} identifies the value of the variable $v$ with respect to timestamp $t$, which is the latest time an agent saw $v$ in the given sequence $\vec{s}$. 
The formal definition they gave is as follows:

\begin{definition}[Retrieval Function]
\label{def:jp:R}
Given a sequence of states  $\seq$, a timestamp $t$, and a variable $v$, the retrieval function, $\memorization: \seqspace \times \mathbb{Z} \times V \rightarrow \mathbb{D}$, is defined as:
    \[
    \memorization(\seq,t,v) = 
\begin{cases} 
\seq[\max(\text{LT})](v) & \text{if } \text{LT} \neq \{\} \\
\seq[\min(\text{RT})](v) & \text{else if } \text{RT} \neq \{\} \\
\none & \text{otherwise}
\end{cases}
    \]
    where LT $= \{j \mid v \!\in \!\seq[j] \ \land \ \seq[j](v) \!\neq \none \ \land \ j \!\leq t \}$ and \\RT $= \{j \mid v \!\in \!\seq[j] \ \land \ \seq[j](v) \!\neq \none \ \land \ t \!< \!j \!\leq \! |\seq|\}$. \defnormalend
\end{definition}

Their intuition is:
if $v$ is in the previous (or current) states in the given justified perspective, then they assume the $v$ stays unchanged since $\max(\text{LT})$, which is the most recent time $v$ in the given perspective before (or at) $t$;
else if $v$ has not been seen before (and at) $t$, then they assume the $v$ stays unchanged to $\min(\text{RT})$, which is the closest time $v$'s value revealed in the given perspective after $t$;
otherwise, $v$ is $\none$, as the given perspective does not contain $v$ at all.

A \textbf{justified perspective}
is a function that:
the input $\seq$ represents the sequence of states of a plan from a particular perspective, which could be an agent's perspective or the global perspective;
and, the output is a sequence of complete local states that $i$ believes, which is $i$'s justified perspective.
It contains $i$'s observation of the input sequence, as well as $i$'s memory, which both can be generated by function $R$ defined above.
They give the formal definition as below: 

\begin{definition}[Justified Perspective Function]
    \label{def:jpm:R}\label{def:jp:function}
Given the input state sequence $\seq$ as $[s_0,\dots,s_n]$, a \emph{Justified Perspective} (JP) function for agent~$i$, $\f_i: \vec{\statespacecomplete} \rightarrow \vec{\statespacecomplete}$, is defined as follows:
    \[
    \f_i([s_0,\dots,s_n])=[s'_0,\dots,s'_n] 
    \]
    where for $t \in [0, n]$ and $v \in V$:
    \[
        \begin{aligned}
            lt_v & = \max(\{j \mid v \in \observation_i(s_j) \land j \leq t \} \cup \{ -1\}) \text{,} & (1)\\
            e    & = \memorization([s_0,\dots,s_t],lt_v,v), & (2)\\
            s''_t & = \{v \assign e \mid s_t(v)=e \lor v \notin \observation_i(s_t\langle \{v \assign e\} \rangle) \}, & (3)\\
            s'_t  & = s_{\none} \langle s''_t \rangle. & (4)
        \end{aligned} \defmathend
    \]
\end{definition}

Line (4) ensures the output is a complete-state sequence, in which the missing variables are filled in with none value ($\none$) assignments.
The value $lt_v$ is the last timestamp the agent $i$ sees $v$, including $-1$ to ensure it is a non-empty set, where $-1$ means this has not been seen. 
Then, the value of $v$ that agent $i$ saw (or should have seen) is retrieved by function $R$ in Line (2) from the above definition. 
Then, Line (3) forms a justified state at timestamp $t$, while, in the meantime, $v \notin \observation_i(s_t\langle \{v \assign e\} \rangle) $ ensures $i$'s memory is consistent with what $i$ sees at the current timestamp $t$.
That is, $s_t\langle \{v \assign e\} \rangle)$ is the state $s_t$ with the value of $v$ overridden with $e$, which may not be $v$'s value in $s_t$. 
Therefore, if agent $i$ observed $v=e$ in the past, then $v=e$ is its current belief if $v \notin \observation_i(s_t\langle \{v \assign e\} \rangle)$. 
If $v=e$ in $s_t$ \emph{and} agent $i$ would see it if it was, then if it does not see $v=e$ in $s_t$, it must be that $v \neq e$.

\textbf{Semantics} are also provided by them in the format of complete semantics and ternary semantics.
Here, we give their ternary semantics\footnote{Because only the ternary semantics is used in the experiments of this paper, while the complete semantics are provided in the supplementary material.} as follows:

\begin{definition}[Ternary semantics]
\label{def:jp:ternary}
Given a JP model instance $M$, state sequence $\seq$ and any epistemic formula in Language $L_{KB}(\Sigma)$, the ternary function $T$ is defined as, omitting model $M$ for readability:

    \begin{supertabular}{@{}ll@{~}l@{~~}l@{~~}r}
      (a) & $T[\seq, r(V_r)]$ & $=$ & 1 if $\pi(\seq[n], r(V_r)) = true$;\\
                 &                          &     & 0 else if $\pi(\seq[n], r(V_r)) = false$;\\
                 &                          &     & $\unknown$ otherwise\\[1mm]
                (b) & $T[\seq, \varphi \land \psi]$ & $=$ & $\min(T[\seq, \varphi], T[\seq, \psi])$\\[1mm]
                (c) & $T[\seq, \neg \varphi]$    & $=$ & $1 - T[\seq, \varphi]$\\[1mm]
                (d) & $T[\seq, S_i v]$           & $=$ & $\unknown$ if $ v \notin \seq[n]$ or $ i \notin \seq[n]$\\[0.5mm]
                 &                          &   & $0$ else if $v \notin \observation_i(\seq[n])$\\
                 &                          &   & $1$ otherwise     \\[1mm]
                (e) & $T[\seq, S_i \varphi]$  & $=$ & $\unknown$ if $T[\seq,\varphi] = \unknown$ or $ i \notin \seq[n]$;\\
                 &                          &   & $0$ else if $T[\observation_i(\seq), \varphi] = \unknown$;\\
                 &                          &   & $1$ otherwise\\[1mm]
                (f) & $T[\seq, K_i \varphi]$ & = & $ T[\seq, \varphi \land S_i\varphi]$\\[1mm]
                (g) & $T[\seq, B_i \varphi]$ & = & $ T[\f_i(s_{\none} \override{\seq}), \varphi]$ \\[1mm]
                \multicolumn{4}{l}{where $n$ is the last (current) timestamp of $\seq$.} & $\blacksquare$\\
    \end{supertabular}

\end{definition}

The ternary semantics employs three truth values: 0 (false), 1 (true), and $\frac{1}{2}$ (unknown). 
Their definition of the ternary semantics allows for arbitrary nesting. 
In Item (g), $s_{\none} \override{\seq}$ ensures that the input to the JP function becomes a sequence of complete states (could be partial due to Item (e)) by filling in the missing variables in the partial states with the none value $\none$.
Using sequence $\vec{s'}\assign[s_0,s_1]$ in Figure~\ref{fig:plan_1} as an example, although $T[\vec{s'}, S_b S_a n]=0$, $T[\vec{s'}, B_b S_a n]=1$ ($T[\observation_b(\vec{s'}),S_a n]=\unknown$ since $n \notin \observation_b(\vec{s'})[1]$).
That is, agent $b$ does not directly see that agent $a$ sees $n$ (knowledge). 
But with the belief of $n$'s existence, $b$ believes that $a$ sees $n$.

Moreover, \citeauthor{DBLP:conf/aips/Hu0L23}~\shortcite{DBLP:conf/aips/Hu0L23} demonstrate that the evaluation time for the ternary semantics is polynomial with respect to search path length and state size, and linear in relation to epistemic formula depth.
They proved their semantics follow the axiomatic system \textbf{KD45}.

By defining observation functions for each agent, \citeauthor{DBLP:conf/aips/Hu0L23}~\shortcite{DBLP:conf/aips/Hu0L23} implemented their own planner following F-STRIPS~\cite{geffner2000functional}. 
Their results show that it is  state-of-the-art in most domains, except the ones with large branching factors, such as Grapevine~\cite{DBLP:journals/ai/MuiseBFMMPS22}, due to the na\"ive blind search algorithm used.

Compared to their previous perspective model in the PWP approach, which was purely based on the observation of the current state, JP model forms perspectives also based on previous observations.
However, JP model only models agents' nested belief under multi-agent settings, not group belief. 
We define group belief in the current paper.

\section{Group Justified Perspective Model}
\label{sec:model}
In this section, we formally propose our group justified perspective (GJP) model by adding group operations for uniform belief, distributed belief, and common belief to the JP model~\cite{DBLP:conf/aips/Hu0L23}, inheriting the existing group modal operators from the PWP approach~\cite{DBLP:journals/jair/Hu0L22}.


The signature of our model $\signature$ is the same as in Definition~\ref{def:pwp:signature}, and the language for group knowledge and belief ($\L_{G\!K\!B}$) is defined by the following grammar:
\begin{definition}[Language]
\label{def:gjp:language}
Given a signature $\Sigma$, the language $\L_{G\!K\!B}(\Sigma)$ is defined by the grammar:
\[
    \begin{array}{@{~}l@{~}l@{~}l}
        \varphi  ::= & \relationbasecase \mid \neg \varphi \mid \varphi \land \varphi \mid S_i v \mid S_i \varphi \mid K_i \varphi \\[1mm]
        & \mid \es_G \varphi \mid \ds_G \varphi \mid \cs_G \varphi \mid \ek_G \varphi \mid \dk_G \varphi \mid \ck_G \varphi \\[1mm]
        & \mid B_i \varphi \mid \eb_G \varphi \mid \db_G \varphi \mid \cb_G \varphi\\[1mm]
    \end{array}
\]
\noindent where $r \in \relationset$, $V_r \subseteq V$ are the terms of $r$ and $r(V_r)$ form a predicate, $v\in V$, and $i \in Agt$. \defnormalend
\end{definition}

Group seeing operators, $\es$, $\ds$ and $\cs$, and knowledge operators, $\ek$, $\dk$ and $\ck$ are from the PWP model~\cite{DBLP:journals/jair/Hu0L22}, while $B_i$ is from the JP model~\cite{DBLP:conf/aips/Hu0L23}.
In this paper, we only focus on introducing the group belief operators, $\eb_G \varphi$, $\db_G \varphi$ and $\cb_G \varphi$, to represent that agents in $G$ uniformly, distributedly and commonly believe $\varphi$ respectively.

\noindent\textbf{Functions:}\quad The observation function $O_i$ and the JP function $\f_i$ are the same as in 
Section~\ref{sec:background:pwp} and Section~\ref{sec:background:jpm}.

\noindent\textbf{Semantics:}\quad We inherit the ternary semantics from the PWP model and the JP model.
The items (a)-(g) follow Definition~\ref{def:jp:ternary}. 
The items (h)-(o) are the group ternary semantics from \cite{DBLP:journals/jair/Hu0L22}, which can be found in the supplementary material, since they are not relevant to the contribution of this paper.
Thus, new semantics start from (p) as in Definition~\ref{def:gjp:eb_ternary}.

Later in this section, we define group justified perspective functions for uniform belief, distributed belief, and common belief, and add ternary semantics for them.
In addition, we validate our semantics on the standard \textbf{KD45} axioms. 
Due to this part being less interesting to the general reader, we put theorems and proofs in the supplementary material.


\subsection{Uniform Belief}
\label{sec:model:uniform}

Uniform Belief is straightforward.
Since a uniform belief of $\varphi$ is that everyone in the group believes $\varphi$, the uniform justified perspective function is just a set union of everyone's individual justified perspectives.

\begin{definition} (Uniform Justified Perspectives)
\label{def:gjp:ef}
  \[
  \ef_G(\vec{s}) = \textstyle\bigcup_{i \in G} \{f_i(\vec{s}) \} \defmathend
  \]
\end{definition}

\begin{definition}[Ternary Semantics for Uniform Belief] 
\label{def:gjp:eb_ternary}
Given an instance of GJP model (omitting it for readability), uniform belief $\eb_G$ for group $G$ is defined:
\[
\text{(p): }\ \T[\vec{s}, \eb_G \varphi] = \min\left(\left\{\T[ \vec{g}, \varphi] \mid \vec{g} \in \ef_G(s_{\none} \override{\vec{s}}) \right\}\right) \tag*{\text{$\blacksquare$}}
\]
\end{definition}

The ternary value of $\T[\vec{s}, \eb_G \varphi]$ depends on the agent that holds the most conservative beliefs of $\varphi$.

Although the ternary semantics of an individual's belief in Definition~\ref{def:jp:ternary} is a \textbf{KD45} semantics, the uniform belief does not follow Axiom \textbf{4} (Positive Introspection) and \textbf{5} (Negative Introspection).
Intuitively, every agent in the group believes $\varphi$ does not guarantee that every agent in the group believes others believe $\varphi$.
That is, $\eb$ follows Axiom \textbf{K} and \textbf{D}.



\subsection{Distributed Belief}
\label{sec:model:distributed}


Distributed Belief is more challenging compared to distributed knowledge.  
The Knowledge Axiom \textbf{T} ($K_i \varphi \Rightarrow \varphi$), which states that knowledge must be true, does not hold for belief. 
This means that agents can hold incorrect beliefs.
If we simply take the distributed union of the perspectives for all agents $i\in G$, the generated set could contain conflicting (inconsistent) beliefs.
To ensure consistency, we form the group distributed justified perspective instead of just uniting each agent's justified perspective.
Intuitively, agents follow their own observations and ``listen'' to agents that have seen variables more recently.
The distributed perspective function $\df$ is defined as follows.

\begin{definition}
[Distributed Justified Perspectives]
    \label{def:gjp:df}
    The distributed justified perspective function for a group of agents \(G\) is defined as follows:
    \[
    df_G([s_0,\dots,s_n])=[s'_0,\dots,s'_n] 
    \]
    where for all \(t \in [0, n]\) and all \(v \in \dom(s_t)\):
    \[
    \begin{aligned}
        lt_v & \assign \max(\{j \mid v \in \textstyle\bigcup_{i \in G} \observation_i(s_j) \land j \leq t \} \cup \{ -1\}) \text{,} & (1)\\
        e    & \assign  \memorization([s_0,\dots,s_t],lt_v,v), & (2)\\
        s''_t & \assign  \{v \assign e \mid s_t(v) \assign e \lor v \notin \textstyle\bigcup_{i \in G} \observation_i(s_t\langle \{v \assign e\} \rangle) \}, & (3)\\
        s'_t  & \assign  s_{\none} \langle s''_t \rangle. & (4)\\
    \end{aligned} \defmathend
    \]
\end{definition}

In this definition, the group distributed justified perspective follows everyone's observation and uses the retrieval function $\memorization$ (in Definition~\ref{def:jpm:R}) to identify the value of the variables that are or were not seen by any agent from the group. 
Intuitively, given any agent $i$ in the group, the value from $i$'s observation in timestamp $t$, $\observation_i(s_t)$, which leads to knowledge, must be true (Axiom T) in $s_t$.
While the value of an unseen variable is determined by anyone in the group that saw it last.
To be specific, the last timestamp the group sees $v$, $lt_v$, is determined by the group observation (formed by union), and then, the value $e$ is retrieved by identifying the closest value that is consistent with it. 
Line (3) ensures the ``group memory'' is consistent with the group observation, while Line (4) ensures the group justified perspective is a sequence of complete states.
So, this definition mimics the definition of the JP function from Definition~\ref{def:jp:function}, except that the variable's value in a state $s'_t$ is taken by the agent(s) that have the most recent view of it.

\begin{figure}[t]
    \centering
    \includegraphics[width=0.47\textwidth]{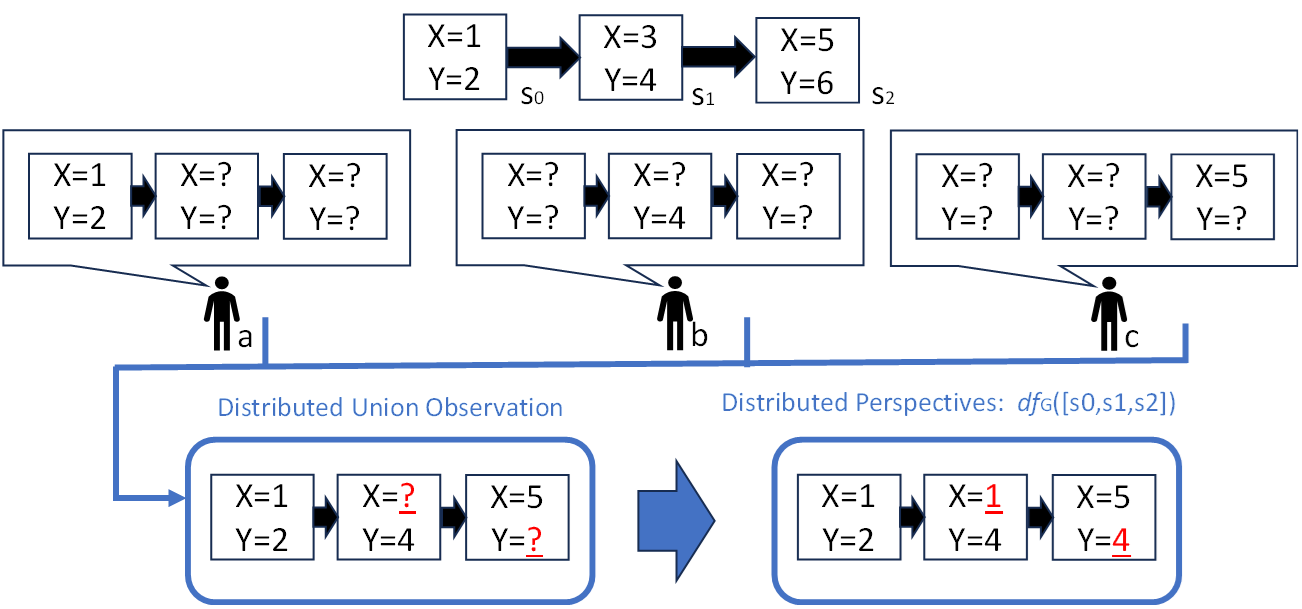}
    \caption{ State sequence $\vec{s}$ and $df_G(\vec{s})$ in Example~\ref{example:distributed}}
    \label{fig:df}
\end{figure}

\begin{example}
\label{example:distributed}
    Let the set of variables be $V \assign \{x,y\}$, domains be $D_x \assign D_y \assign \{1,\dots,6\}$, and a state\footnote{We use the shorthand $m\text{-}n$ to represent the state $\{x \assign m,y \assign n\}$ and $\_$ to represent a missing variable.} sequence be $\vec{s}  \assign  [s_0,s_1,s_2] \assign [1\text{-}2,3\text{-}4, 5\text{-}6]$. 
    Assume $a$ sees $x$ and $y$ in $s_0$, while $b$ sees $y$ in $s_1$ and $c$ sees  $x$  in $s_2$. So, $\observation_a(\vec{s})$=[$1\text{-}2$, $\_\text{-}\_$, $\_\text{-}\_$], $\observation_b(\vec{s})$=[$\_\text{-}\_$, $\_\text{-}4$, $\_\text{-}\_$] and $\observation_c(\vec{s})$=[$\_\text{-}\_$, $\_\text{-}\_$, $5\text{-}\_$]. 
    This is visualized in Figure~\ref{fig:df}.
\end{example}

Intuitively, we can see from Example~\ref{example:distributed} that forming distributed belief is about finding the observation from each agent and deducing the value of group unseen variables, following the same intuition as the JP Model (In Section~\ref{sec:background:jpm}).
Missing values in group observations (noted as ``\textcolor{red}{\underline{?}}'') are retrieved from the ``group memory'' (previous group observation), equating to retrieval from the agent who last observed this value. 
Thus, $\df_G(\seq)=[1\text{-}2,1\text{-}4,5\text{-}4]$.


\begin{definition}[Ternary Semantics for Distributed Belief]
\label{def:gjp:db_ternary}
The distributed ternary semantics are defined using function $\T$, omitting the model $M$ for readability:
\[
\text{(q): } \ \T[\vec{s}, \db_G \varphi] = \T[ \df_G(s_{\none} \override{\vec{s}}), \varphi] \defmathend
\]
\end{definition}

This semantics guarantees that the group distributed justified belief is consistent.
That is done by only merging agents' observations into the group distributed observation, which was consistent with the global state sequence, and deducing the value of unseen variables from it.
This definition is particularly nice as many existing definitions of distributed belief require us to define preference relations over e.g.\ agents or states, to resolve conflicts; see e.g. \cite{liau2003belief}. 
In our definition, the preference relation is implicit and prefers more recent observations over older ones.

At last, we validate the above semantics is a \textbf{KD45} semantics.
The idea of the proof is by having one imaginary agent $x$ whose observation function is the union of the observations of the whole group $G$ ($\observation_x(s)= \textstyle\bigcup_{i\in G} \observation_i(s)$).
Then, by proving $x$'s function $\observation_x$ that is contractive, idempotent, and monotonic, following proofs in \citeauthor{DBLP:conf/aips/Hu0L23}, the ternary semantics (Definition~\ref{def:gjp:db_ternary}) follows \textbf{KD45}.

\subsection{Common Belief}
\label{sec:model:common}
Common belief is the infinite nesting of belief.
Our definition avoids having to calculate the infinite regression by calculating the fixed point of the group's perspectives.

\begin{definition}[Common Justified Perspectives]
\label{def:gjp:cf}
Given a set of perspectives (that is, a set of sequences of states) $\vec{S}$, the common justified perspective is defined as:
  \[
  \cf_G(\vec{S})  \!\assign\!  
  \begin{cases}
       \bigcup_{\vec{s}  \in \vec{S}} \ef_G(\vec{s}) & \text{if } \bigcup_{\vec{s}  \in \vec{S}} \ef_G(\vec{s})  \!\assign\!  \vec{s}\\
          \cf_G(\bigcup_{\vec{s}  \in \vec{S}} \ef_G(\vec{s})) & \text{otherwise.}
    \end{cases} \defmathend
   \]
\end{definition}

The function applies a set union on the uniform perspectives of the group for each input perspective.
Then, the common perspective function repeatedly calls itself by using the output of one iteration as the input of the next iteration, until the input set and output set are the same, which means a convergence of the common perspectives.
Semantically speaking, each iteration adds one level deeper nested perspectives of everyone's uniform belief for evaluation on whether everyone in the group believes. 

\begin{definition}[Ternary Semantics for Common Belief]
\label{def:gjp:cb_ternary}
The group ternary semantics are defined using function $\T$, omitting the model $M$ for readability:
\[
\text{(r): } \ \T[\vec{s}, \cb_G \varphi] \assign \min(\{\T[ \vec{g}, \varphi] \mid \vec{g}  \!\in\! \cf_G(\{s_{\none} \override{\vec{s}} \}) \}) \defmathend
\]
\end{definition}


The common justified perspectives function $\cf_G$ contains the fixed point of all agents' perspectives, their perspectives about others' perspectives, and so on to infinite depth. 
Although the depth is infinite, the definition of $\cf_G$ converges in finite iterations (proof in supplementary material):

\begin{theorem}
\label{thm:fixed-point}
Given a state sequence $\seq$, the iterations needed for $\cf_G(\{\vec{s}\})$ to converge are bounded above $2^{|V|\times |\seq|}$.
\end{theorem}

Although in the worst-case scenario, the maximum number of iterations is $2^{|V|\times |\seq|}$, practically, in our experiments, we find that it converges after a few iterations (Section~\ref{sec:exp}\!).

An example for group justified perspective functions is provided using the same problem in Example~\ref{example:number} as follows:
\begin{example} 
\label{example:groupp}
    Let us use Plan~\ref{plan:1} and let $G  \!\assign\!  \{a,b\}$.
    The sequences of the states can be represented by a list of states\footnote{We use shorthand $\text{t-}\text{f-}2$ to represent $\{peeking_a \!\assign\! true, peeking_b \!\assign\! false, n \!\assign\! 2\}$ is represented as $\text{t-}\text{f-}2$.} as: 
    $\seq  \assign  [\text{f-}\text{f-}2,\text{t-}\text{f-}2,\text{f-}\text{f-}2,\text{f-}\text{f-}1,\text{f-}\text{t-}1]$.
\end{example}

Then, the common justified perspective $\cf_G(\{\vec{s}\})$ of a group $G$ in the above example  converged within 3 iterations (a step-by-step illustration is provided in supplementary material). 
The output is $\{[\text{f-}\text{f-}\!\none,\text{t-}\text{f-}2,\text{f-}\text{f-}2,\text{f-}\text{f-}2,\text{f-}\text{t-}2]$, $[\text{f-}\text{f-}\!\none,\text{t-}\text{f-}\!\none,\text{f-}\text{f-}\!\none,\text{f-}\text{f-}\!\none,\text{f-}\text{t-}1]$, $[\text{f-}\text{f-}\!\none,\text{t-}\text{f-}\!\none,\text{f-}\text{f-}\!\none,\text{f-}\text{f-}\!\none,\text{f-}\text{t-}2]\}$.
According to the ternary semantics of the $\cb$ operator, $\T[\vec{s},\cb_G n<3] \!\assign\! \min(\{\T[\vec{g}, n<3] \mid \vec{g} \!\in\! \cf_G(\{\seq\})\}) \!\assign\! \min(\{1,1,1\}) \!\assign\! 1$.
Therefore, Plan~\ref{plan:1} can form a common belief for the group $G \!\assign\! \{a,b\}$ that $\cb_G n<3$.
At last, we validate that the ternary semantics of the $\cb$ operator is a \textbf{KD4} semantics.

To sum up, this section presents semantics for group uniform, distributed, and common belief.
Given that uniform belief evaluates each agent's justified perspective, while distributed belief synthesizes a justified perspective from a group's collective observations, the time complexity of these is polynomial, specifically scaled by the number of agents, analogous to the JP model~\cite{DBLP:conf/aips/Hu0L23}.
Theorem~\ref{thm:fixed-point} shows that the worst-case time complexity for common belief is exponential, factoring in the iterations to identify the fix-point set (common justified perspectives) of individual justified perspectives.
\begin{table*}[!th]
  \scriptsize
  \centering
  \begin{tabular}{lrrrrrrrrrrl}
    \toprule
    \multirow{2}{*}{ID} & \multirow{2}{*}{Gen} & \multicolumn{2}{c}{$n$ in $\cf^n$} & \multicolumn{2}{c}{$|\cf|$} & \multicolumn{3}{c}{External} & Total & \multirow{2}{*}{$|p|$} & \multirow{2}{*}{Goals} \\
    & & Max & Avg & Max & Avg & \#$\cf$ & \#$calls$ & $\overline{\text{T}}$ (ms) & T (s) & & \\
    \midrule
    N0 & 141 & 0 & 0 & 0 & 0 & 0 & 141 & 0.10 & 0.04 & 4 & $\eb_{G}\ n<2$ \\
    N1 & 26 & 0 & 0 & 0 & 0 & 0 & 26 & 0.13 & 0.01 & 2 & $\db_{G}\ n<2$ \\
    N2 & 141 & 4 & 2.20 & 5 & 2.16 & 141 & 141 & 0.43 & 0.09 & 4 & $\cb_{G}\ n<2$ \\
    N3 & 141 & 4 & 1.87 & 5 & 1.86 & 446 & 141 & 0.91 & 0.15 & 4 & $\cb_{G} \cb_{G}\ n<2$ \\
    N4 & 141 & 4 & 1.76 & 5 & 1.75 & 971 & 141 & 1.70 & 0.26 & 4 & $\cb_{G} \cb_{G} \cb_{G}\ n<2$ \\
    N5 & 112 & 3 & 2.13 & 4 & 2.04 & 112 & 112 & 0.42 & 0.07 & 4 & $\neg \eb_{G}\ n \assign 1 \land \neg \eb_{G}\ n \assign 2 \land \cb_{G}\ n<2$ \\
    N6 & 178 & 3 & 1.65 & 3 & 1.65 & 356 & 178 & 0.48 & 0.12 & 4 & $B_a \cb_{G}\ n \assign 2 \land B_b \cb_{G}\ n \assign 1$ \\
    \midrule
    G0 & 36 & 4 & 3.03 & 7 & 5.36 & 36 & 41 & 2.76 & 0.13 & 1 & $\cb_{G}\ sct_a  \assign  t$ \\
    G1 & 36 & 4 & 2.10 & 7 & 2.69 & 229 & 41 & 7.51 & 0.32 & 1 & $\cb_{G}\cb_{G}\ sct_a  \assign  t$ \\
    G2 & 36 & 4 & 1.79 & 7 & 2.17 & 652 & 41 & 16.00 & 0.67 & 1 & $\cb_{G}\cb_{G}\cb_{G}\ sct_a  \assign  t$ \\
    G3 & 36 & 4 & 1.66 & 7 & 1.99 & 1451 & 41 & 32.78 & 1.36 & 1 & $\cb_{G}\cb_{G}\cb_{G}\cb_{G}\ sct_a  \assign  t$ \\
    G4 & 1474 & 4 & 3.38 & 10 & 6.74 & 1474 & 1860 & 5.52 & 11.12 & 4 & $\eb_{G}\ sct_a  \assign  t \land \neg \cb_{G}\ sct_a  \assign  t$ \\
    G5 & 8997 & 5 & 3.50 & 14 & 8.30 & 8997 & 13350 & 10.63 & 151.92 & 6 & $\eb_{G} \eb_{G}\ sct_a  \assign  t \land \neg \cb_{G}\ sct_a  \assign  t$ \\
    G6 & 914 & 4 & 3.02 & 8 & 4.21 & 1828 & 1138 & 4.74 & 5.91 & 3 & $B_b \cb_{G}\ sct_a  \assign  f \land \cb_{\{a,c,d\}} sct_a  \assign  t$ \\
    G7 & 2960 & 4 & 2.66 & 8 & 3.18 & 14800 & 3792 & 10.23 & 40.72 & 4 & $\cb_{\{b,c\}} \cb_{G}\ sct_a  \assign  f \land \cb_{\{a,d\}} sct_a  \assign  t$ \\
    G8 & 1474 & 4 & 3.38 & 10 & 6.74 & 1474 & 1860 & 6.84 & 13.61 & 4 & $\db_{G} \eb_{G}\ sct_a  \assign  t \land \neg \cb_{G}\ sct_a  \assign  t$ \\
    G9 & 1530 & 0 & 0 & 0 & 0 & 0 & 1926 & 1.23 & 3.31 & 4 & $\db_{G} \eb_{G}\ sct_a  \assign  t \land \neg B_a \eb_{G}\ sct_a  \assign  t$ \\

    \bottomrule
  \end{tabular}

  \caption{Result N0--N6 and G0--G9 are instances for Number and Grapevine domain respectively. 
  $G$ represents the group of all agents -- $\{a,b\}$ for Number; and $\{a,b,c,d\}$ for Grapevine. 
  ``Gen'' is the number of search nodes generated during search. 
  ``Max'' and ``Avg'' under ``$n$ in $\cf^n$'' are the maximum and average number of iterations for each $\cf$, 
  and ``Max'' and ``Avg'' under $|\cf|$ represent the maximum and average size of converged $\cf$.
  ``\#$\cf$'' under ``External'' represents the number of $\cf$ function calls (could be more than the external function calls due to nested common belief or multiple common beliefs). 
  ``\#$calls$'' and ``$\overline{\text{T}}$ (ms)'' are the number and average time of external function calls. 
  ``T (s)'' is total runtime, and $|p|$ is the plan length.}
  \label{tab:all}
\end{table*}


\section{Experiments}
\label{sec:exp}
Since there are no planning benchmarks for group belief, we select two domains (Number and Grapevine) from existing work~\cite{DBLP:journals/jair/Hu0L22,DBLP:conf/aips/Hu0L23} and add several challenging instances that use group belief
, including instances with inconsistent or nested group beliefs.

The source code of the planner, the domain, the problem, and external function files, as well as experimental results, are downloadable from:
\textbf{omitted to anonymity}.
%
We extend the F-STRIPS planner from \cite{DBLP:conf/aips/Hu0L23}.
To demonstrate the efficiency of our model instead of the particular search algorithms, we use the BrFS (breadth-first search) search algorithm with duplicate removal.
The experiments are run on a Linux machine (Ubuntu 20.04) with 8 CPUs (Intel i7-10510U 1.80GHz) and 16GB RAM.
Similar to the JP model~\cite{DBLP:conf/aips/Hu0L23}, the external functions, implemented in Python, evaluate the belief formulae (either in goals or in actions' preconditions) following our new semantics defined in Section~\ref{sec:model} when search nodes are generated.


\textbf{Number} is the same as described in Example~\ref{example:number}, while 
\textbf{Grapevine} is a benchmark domain~\cite{DBLP:journals/ai/MuiseBFMMPS22}.
In two adjacent rooms, 4 agents (in the same room), each have their own secret (For simplicity, we only consider agent $a$'s secret $sct_a \!\assign\! t$).
All agents can move between two rooms and \emph{share} or \emph{lie} about a secret, if either the secret is their own or they have heard the secret.
That is, they need to have a valid belief about the secret ($B_i\ sct_j\!\!\neq \none$) before they can share it.

\subsection{Results}
\label{sec:exp:result}
The results can be found in Table~\ref{tab:all}. 
All group beliefs, except common belief, can be evaluated easily, which is indicated by $\overline{\text{T}}$ (N0, N1, and G9).
It takes longer to evaluate common belief as it requires the common perspective set to converge; 
however, the number of iterations for $\cf$ to converge is around $1.65-3.50$ -- much less than the worst-case identified in Theorem~\ref{thm:fixed-point}.
This is because, in practical epistemic benchmarks, the number of justified perspectives (beliefs) is bounded by the actual state sequence and the difference between each single agent's nested perspectives, resulting in relatively small converged sizes (as shown by $|\cf|$). 
In addition, N2-N4 and G0-G4 have the exact same search trees (as their goals are semantically equivalent due to $\cb$ satisfying Axiom~\textbf{4}), while when finding the converged common perspective set for nested common belief (using G2 as an example), the growing rate of \#$cf$ (excluded \#$\cf$ in G0 and G1) is approximately the average of $|\cf|$ for the previous nesting depth (for \#$\cf$ in G2: $229+(229-36)\times \overline{|\cf^2|} = 651.67 \approx 652$, where $\overline{|\cf^2|} = (229 \times 2.69 - 36 \times 5.36) \div (229-36)=2.19$) instead of the number of agents.
This growing rate often keeps decreasing (larger than $1$) when the common perspective sets not only contain the actual global perspective ($\cf({\seq})=\{\seq\} \rightarrow  \overline{|\cf^n|}=1$, the extreme lower-bound results in \#$\cf$ being linear) in all sequences that occurred during the search.
This can be verified with Example~\ref{example:groupp} as well.


Semantically, it is worth noting that everyone believes does not result in a common belief (G4). 
Even that everyone believes that everyone believes $a$'s secret is true does not form a common belief of $s$'s secret is true (G5).




\section{Conclusion and Future work}
\label{sec:conclusion}
In this paper, we define an extension to the JP model to handle group beliefs; 
implement its ternary semantics as an action-model-free planning tool;
discuss and show the axiomatic system each group belief operator follows;
and demonstrate its expressiveness and efficiency on newly featured domains.
The results show that our approach can effectively handle multi-agent epistemic planning problems with group beliefs and do so efficiently, even with a simple prototype F-STRIPS planner implementing the BrFS.

For future work, we will implement efficient search algorithms for our planner.
Novel non-Markovian search algorithms may be necessary, as our GJP model works with state sequences.
In addition, although there is a growing body of research in epistemic planning, there are still no uniform standards on either epistemic planning language or benchmark domains.
It would be valuable to revisit existing approaches and model our benchmarks in those.
Future research could broaden the field's applicability by relaxing assumptions based on classical planning, particularly in dynamic environments and in human-agent interaction domains with a formal human belief model.

\appendix
\section{Examples}

Here, we provide detailed explanations for some difficult examples in our formalization in Section~\ref{sec:model}.

\subsection{Distributed Perspective Function}

For the example (Example~\ref{example:distributed} and Figure~\ref{fig:df}) of distributed perspective function (Definition~\ref{def:gjp:df}), the conclusion can be worked out as follows:

Following the formalization in Definition~\ref{def:gjp:df}, $df_G(\seq)[0]=1\text{-}2$ is trivial (both $x$ and $y$ are in $\bigcup_{i\in G}\observation_i(s_0)$ because of $a$).
When $t$ is 1,  $df_G(\seq)[1](y)=4$ trivially holds for the same reason, while $lt_x$ (Line 1) is still $0$ because none of the agents see $x$ at $s_1$.
Therefore, the retrieved value $R([s_0,s_1],0,x)=s_0(x)=1$ (Line 2). 
$x=1$ is in $s_1''$ ($s_1''=1\text{-}4$) since that even if the global state updated from $3\text{-}4$ to $1\text{-}4$, $x$ is still not in anyone's observation $\bigcup_{i\in G}\observation_i(1\text{-}4)$ (Line 3).
At last, when $t$ is 2, $df_G(\seq)[2](x)=5$ holds since $c$ sees it, while $lt_y=max(\{0,1\}\cup \{-1\})$ (Line 1) is 1 ($a$ sees it in $s_0$ and $b$ sees it in $s_1$).
$s_2''$ is $5\text{-}4$ ($R([s_0,s_1,s_2],1,y)=4$).
Thus, $\df_G(\seq)=[1\text{-}2,1\text{-}4,5\text{-}4]$.

\subsection{Common Perspective Function}

For clarity, we inherit the following theorem, which is proposed and proved by \citeauthor{guang2025thesis}~\shortcite{guang2025thesis}.
\begin{theorem}
For any agent $i \in Agt$ and perspective $\Vec{s} \in \Vec{S}$:
\label{thm:fieqfifi}
    \[
        \f_i(\vec{s}) = \f_i(\f_i(\vec{s}))
    \]
\end{theorem}

For the example (Example~\ref{example:groupp} from Example~\ref{example:number} and Figure~\ref{fig:plan_1}) of common perspective function (Definition~\ref{def:gjp:cf}), the conclusion of the converged set can be worked out as follows:

\textbf{Iteration 1:} $\bigcup_{\vec{g}  \in \vec{S}} \ef_G(\vec{g})  \!\assign\!  \{\f_a(\seq),\f_b(\seq)\}$.
Following Definition~\ref{def:jp:function}, we have:

\begin{tabular}{r@{~}c@{~}l}
    $\f_a(\seq)$ & $\assign$ & $ [\text{f-}\text{f-}\!\none,\text{t-}\text{f-}2,\text{f-}\text{f-}2,\text{f-}\text{f-}2,\text{f-}\text{t-}2]$ \\
    $\f_b(\seq)$ & $\assign$ & $ [\text{f-}\text{f-}\!\none,\text{t-}\text{f-}\!\none,\text{f-}\text{f-}\!\none,\text{f-}\text{f-}\!\none,\text{f-}\text{t-}1]$ \\
\end{tabular}

Since the input set of sequences is different from the output, which means the set has not converged yet, further iteration is needed.

\textbf{Iteration 2:} $\bigcup_{\vec{g} \in \{\f_a(\seq),\f_b(\seq)\}} \ef_G(\vec{g}) \!\assign\! \ef_G(\f_a(\seq))\! \cup\! \ef_G(\f_b(\seq)) \!\assign\! \{\f_a(\f_a(\seq)),\f_b(\f_a(\seq)),\f_a(\f_b(\seq)),\f_b(\f_b(\seq))\}$.
Following Definition~\ref{def:jp:function} and Theorem~\ref{thm:fieqfifi}, we have:

\begin{tabular}{r@{~}c@{~}l}
    $\f_a(\f_a(\seq))$ & $\assign$ & $\f_a(\seq)$ \\
    $\f_b(\f_a(\seq))$ & $\assign$ & $ [\text{f-}\text{f-}\!\none,\text{t-}\text{f-}\!\none,\text{f-}\text{f-}\!\none,\text{f-}\text{f-}\!\none,\text{f-}\text{t-}1]$ \\
    $\f_a(\f_b(\seq))$ & $\assign$ & $ [\text{f-}\text{f-}\!\none,\text{t-}\text{f-}\!\none,\text{f-}\text{f-}\!\none,\text{f-}\text{f-}\!\none,\text{f-}\text{t-}2]$ \\
    $\f_b(\f_b(\seq))$ & $\assign$ & $\f_b(\seq)$ \\
\end{tabular}

Since the output set is $\{\f_a(\seq),\f_b(\seq),\f_a(\f_b(\seq))\}$ ($\f_b(\f_a(\seq))$ is the same as $\f_b(\seq)$), which is still different from the input set, we have not reached a fixed point, so at least one more iteration is needed.

\textbf{Iteration 3:} $\bigcup_{\vec{g}  \in \{\f_a(\seq),\f_b(\seq),\f_a(\f_b(\seq))\}}\! \ef_G(\vec{g})\!\! \!\assign\! \!\!\ef_G(\f_a(\seq)) \allowbreak \cup \ef_G(\f_b(\seq))\! \cup\! \ef_G(\f_a(\f_b(\seq)))\!\! \!\assign\! \!\!\{\f_a(\f_a(\seq)), \f_b(\f_a(\seq)),\allowbreak\f_a(\f_b(\seq)),\f_b(\f_b(\seq)),\f_a(\f_a(\f_b(\seq))),\f_b(\f_a(\f_b(\seq)))\}$.

\begin{tabular}{r@{~}c@{~}l}
    
    $\f_a(\f_a(\seq))$ & $\assign$ & $\f_a(\seq)$ \\
    $\f_b(\f_a(\seq))$ & $\assign$ & $\f_b(\seq)$ \\
    $\f_a(\f_b(\seq))$ & $\assign$ & $ [\text{f-}\text{f-}\!\none,\text{t-}\text{f-}\!\none,\text{f-}\text{f-}\!\none,\text{f-}\text{f-}\!\none,\text{f-}\text{t-}2]$ \\
    $\f_b(\f_b(\seq))$ & $\assign$ & $\f_b(\seq)$ \\
    $\f_a(\f_a(\f_b(\seq)))$ & $\assign$ &  $\f_a(\f_b(\seq))$ \\
    $\f_b(\f_a(\f_b(\seq)))$ & $\assign$ &  $[\text{f-}\text{f-}\!\none,\text{t-}\text{f-}\!\none,\text{f-}\text{f-}\!\none,\text{f-}\text{f-}\!\none,\text{f-}\text{t-}2]$ \\
\end{tabular}

Since $\f_b(\f_a(\f_b(\seq)))$ is the same as $\f_a(\f_b(\seq))$, the output set is $\{\f_a(\seq),\f_b(\seq),\f_a(\f_b(\seq))\}$, which is the same as the input.
Therefore, the common justified perspective has converged.
In addition, according to the ternary semantics of the $\cb$ operator, $\T[\vec{s},\cb_G n<3] \!\assign\! \min(\{\T[\vec{g}, n<3] \mid \vec{g} \!\in\! \cf_G(\{\seq\})\}) \!\assign\! \min(\{1,1,1\}) \!\assign\! 1$.
Therefore, the sequence of action in Plan~\ref{plan:1} can form a common belief for the group $G \!\assign\! \{a,b\}$ that $\cb_G n<3$.
\section{Theorem and Proofs}

Here, we provided the theorem of the claims in Section~\ref{sec:model} and their proofs.

\subsection{Uniform Belief}

For the uniform belief operator, $\eb$, its ternary semantics (Definition~\ref{def:gjp:eb_ternary}) derived from the uniform perspective function $\ef$ (Definition~\ref{def:gjp:ef}) satisfies \textbf{KD} axioms.
\begin{theorem}
    The ternary semantics for the uniform belief operator $\eb$ follows the axiomatic system \textbf{KD}:
    
    \noindent
    \begin{tabular}{@{~}l@{~}l}
    	\textbf{K} & (Distribution):           $\eb_G \varphi \land \eb_G(\varphi \rightarrow \psi) \rightarrow \eb_G \psi $\\[1mm]
    	\textbf{D} & (Consistency):            $\eb_G \varphi \rightarrow \neg \eb_G \neg \varphi $\\[1mm]
    \end{tabular} 
\end{theorem}

\begin{proof}
    Axiom \textbf{K} is straightforward.
    In the form of ternary semantics, $\T[\seq,\eb_G \varphi \land \eb_G(\varphi \rightarrow \psi)]=1$ means $\T[\seq,\eb_G \varphi]=1$ and $\T[\seq, \eb_G(\varphi \rightarrow \psi)]=1$ (item b in Definition~\ref{def:jp:ternary}).
    Since $\ef_G(\seq)$ is a union of everyone's justified perspectives in group $G$, we have $\forall i \in G, \T[\seq,B_i \varphi]=1$ and $\forall i \in G, \T[\seq,B_i (\varphi \rightarrow \psi)]=1$, according to Definition~\ref{def:gjp:eb_ternary}.
    Then, for every agent $i$ in group $G$, we have $\T[\f_i(\seq),\varphi]=1$ and $\T[\f_i(\seq),(\varphi \rightarrow \psi)]=1$.
    As both $\varphi$ and $(\varphi \rightarrow \psi)$ are evaluated as true in $\f_i(\seq)$, $\psi$ also holds in $\f_i(\seq)$.
    Thus, $\T[\seq,\eb_G \psi]=1$.

    For Axiom~\textbf{D}, using similar reasoning as above, $\T[\seq, \eb_G \varphi ]=1$ means $\forall i \in G, \T[\seq,B_i \varphi]=1$.
    For the right side, we have $\T[\seq,\neg \eb_G \neg \varphi]=1-\T[\seq,\eb_G \neg \varphi]$.
    Since $\forall i \in G, \T[\seq,B_i \varphi]=1$, we have $\forall i \in G, \T[\seq, B_i \neg \varphi]=0$.
    Therefore, $\T[\seq,\neg \eb_G \neg \varphi]=1$.
\end{proof}

\subsection{Distributed Belief}

For the distributed belief operator, $\db$, its ternary semantics (Definition~\ref{def:gjp:db_ternary}) derived from the distributed perspective function $\df$ (Definition~\ref{def:gjp:df}) satisfies \textbf{KD45} axioms.

\begin{theorem}
    The ternary semantics for the distributed belief operator $\db$ follows the axiomatic system \textbf{KD45}: 

    \noindent
    \begin{tabular}{@{~}l@{~}l}
    	\textbf{K} & (Distribution):           $\db_G \varphi \land \db_G(\varphi \rightarrow \psi) \rightarrow \db_G \psi $\\[1mm]
    	\textbf{D} & (Consistency):            $\db_G \varphi \rightarrow \neg \db_G \neg \varphi $\\[1mm]
    	\textbf{4} & (Positive Introspection): $\db_G \varphi \rightarrow  \db_G \db_G \varphi $\\[1mm]
    	\textbf{5} & (Negative Introspection): $\neg \db_G \varphi \rightarrow  \db_G \neg \db_G \varphi $\\[1mm]
    \end{tabular} 

\end{theorem}

Different from uniform and common beliefs, this theorem can be proved in a more elegant way by induction.
As mentioned in Section~\ref{sec:background}, \citeauthor{DBLP:conf/aips/Hu0L23}~\shortcite{DBLP:conf/aips/Hu0L23} prove that for any individual agent $i$, with an observation function $\observation_i$ that is contractive, idempotent, and monotonic, the semantics (Definition~\ref{def:jp:ternary}) with the justified perspective function (Definition~\ref{def:jp:function}) is a \textbf{KD45} semantics.
Our definition for the distributed justified perspective function (Definition~\ref{def:gjp:df}) mimics the justified perspective function except the observation is taking a union of every agent's observation from the group.
This is equivalent to having one imaginary agent $x$ whose observation function is the union of the observations of the whole group $G$ ($\observation_x(s)= \bigcup_{i\in G} \observation_i(s)$).
To show the above theorem holds, we just need to prove the observation function for agent $x$ is contractive, idempotent, and monotonic.

\begin{theorem}
\label{thm:ox}
    Given a group $G$ of agents and their observation functions (contractive, idempotent, and monotonic), the union of their observation functions $\observation_x$ ($\observation_x(s)= \bigcup_{i\in G} \observation_i(s)$) is also:
    
    \vspace{1mm}
    \begin{tabular}{@{~}l@{~}l}
    	Contractive: & $\observation_x(s) \subseteq s$\\[1mm]
    	Idempotent: & $\observation_x(s) = \observation_x(\observation_x(s))$\\[1mm]
    	Monotonic: & If $s \subseteq s'$, then $\observation_x(s) \subseteq \observation_x(s')$\\
    \end{tabular} 
\end{theorem}

\begin{proof}   
    Contractiveness and monotonicity are straightforward. 
    Since each $\observation_i(s) \subseteq s$, their union $\observation_x(s)$ is also a subset or equal to $s$.
    Then, for any states $ s \subseteq s'$, since each $\observation_i(s) \subseteq \observation_i(s')$, the union of every agent's $\observation_i(s)$ is also a subset or equal to the union of every agent's $\observation_i(s')$.

    For idempotence, the premise gives $O_i(s) \subseteq O_x(s)$.
    For each agent $i$ in the group, according to monotonicity, we have $O_i(O_i(s)) \subseteq O_i(O_x(s))$, which is $O_i(s) \subseteq O_i(O_x(s))$ (idempotence of $O_i$).
    Applying union on both sides for the group, we have $\bigcup_{i \in G} O_i(s) \subseteq \bigcup_{i \in G} O_i(O_x(s))$, which is $O_x(s) \subseteq O_x(O_x(s))$.
    From the other direction, we have $O_x(s) \subseteq s$ (Contractiveness).
    For each agent $i$ in the group, according to monotonicity, we have $ O_i(O_x(s)) \subseteq O_i(s)$.
    Applying union on both sides for the group, we have  $\bigcup_{i \in G}  O_i(O_x(s)) \subseteq \bigcup_{i \in G} O_i(s)$, which is $O_x(O_x(s)) \subseteq O_x(s)$.
    Combining together, we have $\observation_x(s) = \observation_x(\observation_x(s))$.
\end{proof}

Therefore, Theorem~\ref{thm:ox} holds, which indicates that the semantics for the distributed belief operator is a \textbf{KD45} semantics.

\subsection{Common Belief}

Firstly, we prove Theorem~\ref{thm:fixed-point} (the upper bound for converged common perspective set size) as follows:

    \begin{proof}
        Since for each variable in the last state of a justified perspective $\vec{w}$, its value is either visible (same as its in the last state of the global perspective), or not visible (same as its in the second-last state from $\vec{w}$),
        the number of possible states in each index of a justified perspective is $2^{|V|}$.
        So, the number of possible perspectives given a global state sequence $\vec{s}$ with a length of $n$ is $2^{|V|\times n}$.
        In calculating $\cf_{G}$, either the base case holds (that is, combining the perspective of the group for all $\vec{s}  \!\in\! \vec{S}$ does not change the common perspective), so it terminates and adds no new perspectives; or the recursive step holds. 
        In this case, the input of the $\cf$ function is a set that contains perspectives from each agent in the format of $\vec{S} \!\assign\! \{ \f_i(\vec{s}),\f_i(\vec{s}'),\dots \mid \forall i  \!\in\! G \}$.
        Then, we apply $\f_j$ for each agent $j$ in the group $G$ on each perspective from $\vec{S}$ as $\vec{S}'  \!\assign\!  \bigcup_{\vec{s}  \in \vec{S}} \ef_{G}(\vec{s})$.
        For each $\f_i(\vec{s})$ from $\vec{S}$, we have $\f_j(\f_i(\vec{s}))$ in $\vec{S}'$ for each agent $j$ in group $G$.
        With Theorem~\ref{thm:fieqfifi}, we have $\f_j(f_i(\vec{s}))  \!\assign\!  \f_i(\vec{s})$ when $i \!\assign\! j$.
        Therefore, we have $\vec{S} \subseteq \vec{S}'$. 
        At worst, we add one new sequence each iteration, meaning that $\cf_{G}(\vec{S})$ converges by at most $2^{|V|\times n}$ iterations.
    \end{proof}

Then, for the common belief operator, $\cb$, its ternary semantics (Definition~\ref{def:gjp:cb_ternary}) derived from the common perspective function $\cf$ (Definition~\ref{def:gjp:cf}) satisfies \textbf{KD4} axioms.

\begin{theorem}
    The ternary semantics for the common belief operator $\cb$ follows the axiomatic system \textbf{KD4}: 

    \noindent
    \begin{tabular}{@{~}l@{~}l}
    	\textbf{K} & (Distribution):           $\cb_G \varphi \!\land\!\cb_G(\varphi \!\rightarrow\!\psi) \!\rightarrow\!\cb_G \psi $\\[1mm]
    	\textbf{D} & (Consistency):            $\cb_G \varphi \!\rightarrow\!\neg \cb_G \neg \varphi $\\[1mm]
    	\textbf{4} & (Positive Introspection): $\cb_G \varphi \!\rightarrow\! \cb_G \cb_G \varphi $\\[1mm]
    \end{tabular} 
\end{theorem}

\begin{proof}
    For Axiom \textbf{K}, $\T[\seq, \cb_G \varphi \!\land\!\cb_G(\varphi \!\rightarrow\!\psi)] \!\assign\! 1$ means $\forall \vec{g}  \!\in\! \cf_G(\{\seq\}), \T[\vec{g}, \varphi] \!\assign\! 1$ and $\forall \vec{g}  \!\in\! \cf_G(\{\seq\}), \T[\vec{g}, \varphi \!\rightarrow\!\psi] \!\assign\! 1$.
    Therefore, we have $\forall \vec{g}  \!\in\! \cf_G(\{\seq\}), \T[\vec{g}, \psi] \!\assign\! 1$, which makes $\T[\seq, \cb_G \psi] \!\assign\! 1$.

    For Axiom \textbf{D},  $\T[\seq, \cb_G \varphi] \!\assign\! 1$ means  $\forall \vec{g}  \!\in\! \cf_G(\{\seq\}),$ $\T[\vec{g}, \varphi] \!\assign\! 1$.
    $\T[\seq, \neg \cb_G \neg \varphi] \!\assign\! 1-\T[\seq, \cb_G \neg \varphi]$.
    Since $\T[\seq, \cb_G \neg \varphi] \!\assign\! \min(\{ \T[\vec{g}, \neg \varphi] \mid \vec{g}  \!\in\! \cf_G(\{\seq\})\})$ and $\forall \vec{g}  \!\in\! \cf_G(\{\seq\}), \T[\vec{g}, \neg \varphi] \!\assign\! 0$, we have $\T[\seq, \cb_G \neg \varphi] \!\assign\! 0$, which means  $\T[\seq, \neg \cb_G \neg \varphi] \!\assign\! 1$.

     For Axiom \textbf{4}, the premise $\T[\seq, \cb_G  \varphi] \!\assign\! 1$ means $\forall \vec{g}  \!\in\! \cf_G(\{\seq\}), \T[\vec{g}, \varphi] \!\assign\! 1$.
     According to the definition of function $\cf$, since $ \{\vec{g}\} \subseteq \cf_G(\{\seq\})$, we have $\bigcup_{\vec{t} \in \{\vec{g}\}}\ef_G(\vec{t}) \subseteq  \cf_G(\{\seq\})$.
     In addition, apply any times of $\ef_G$ on this perspective set (etc. $\ef_G\big(\bigcup_{\vec{t} \in \{\vec{g}\}}\ef_G(\vec{t})\big)$ for 2 times), the result must be a subset or equal of $\cf_G(\{\seq\})$.
     This means $\forall \vec{g}  \!\in\! \cf_G(\{\seq\}), \cf_G(\{\vec{g}\}) \subseteq  \cf_G(\{\seq\})$.
     Since $\forall \vec{g}  \!\in\! \cf_G(\{\seq\}), \T[\vec{g}, \varphi] \!\assign\! 1$, we have $\forall \vec{g}  \!\in\!  \cf_G(\{\seq\}), \forall \vec{g'}  \!\in\! \cf_G(\{\vec{g}\}),  \T[\vec{g'}, \varphi] \!\assign\! 1$.
     Thus, Axiom \textbf{4} holds.
\end{proof}

\section{Supplementary Material: Semantics}
Signature (Definition~\ref{def:pwp:signature}), Language (Definition~\ref{def:gjp:language}), and Model (Definition~\ref{def:jp:model}) are defined in the main content.

In order to provide semantics for $CS$, we need to re-introduce the perspective (observation) function $\cc\observation$.

\begin{definition}[Common Observation Function~\cite{DBLP:journals/jair/Hu0L22}]
Given a group of agents $G$ and the current state $s$, the common observation of the group can be defined as:
\[
    \cc\observation(G,s) = 
    \begin{cases}
        s & \text{if } s = \bigcap _{i \in G} \observation_i(s)\\
        \cc\observation(G,\bigcap _{i \in G} \observation_i(s)) & \text{otherwise.}
    \end{cases}
\]
\end{definition}
The variables that are not visible to any agent in the group $G$ are filtered out until the remaining set becomes a fixed point set.
That is, every variable in the set is commonly seen by the group $G$.

\subsection{Ternary Semantics}
Here we provide full ternary semantics.
Similar to the complete semantics:
item a-g~\cite{DBLP:conf/aips/Hu0L23} are for individual modal operators using justified perspective function;
item h-o~\cite{DBLP:journals/jair/Hu0L22} are for group seeing and knowledge operators using group perspective functions;
item p-r are for group belief operators defined in Section~3.2 from the main paper.

\begin{definition}[Ternary Semantics]
\label{def:jpm:ternary}
The ternary semantics are defined using function $\T$, omitting the model $M$ for readability:

\vspace{2mm} 
\noindent
\begin{supertabular}{@{~~}l@{~~}l@{~~}l@{~~}l@{~~}l}
    (a) & $\T[\seq, r(V_r)]$ & $=$ & 1 & if $\pi(\seq[n], r(V_r)) = true$;\\
        &                          &     & 0 & else if $\pi(\seq[n], r(V_r)) = false$;\\
        &                          &     & $\unknown$ & otherwise\\[1mm] 
    (b) & $\T[\seq, \phi \land \psi]$ & $=$ &\multicolumn{2}{l}{ $\min(\T[\seq, \phi], \T[\seq, \psi])$}\\[1mm]
    (c) & $\T[\seq, \neg \varphi]$    & $=$ & \multicolumn{2}{l}{ $1 - \T[\seq, \varphi]$}\\[1mm]
    (d) & $\T[\seq, S_i v]$           & $=$ & $\unknown$ & if $ i \notin \seq[n]$ or $ v \notin  \seq[n]$\\[1mm]
        &                          &   & $0$ & if $v \notin \observation_i(\seq[n])$\\
        &                          &   & $1$ & otherwise     \\
    (e) & $\T[\seq, S_i \varphi]$  & $=$ & $\unknown$ & if $\T[\seq,\varphi] = \unknown$ or $ i \notin \seq[n]$;\\
    &                          &   & $0$ & if $\T[ \observation_i(\seq), \varphi] = \unknown$;\\
    &                          &   & $1$ & otherwise\\[1mm]
    (f) & $\T[\seq, K_i \varphi]$ & = & \multicolumn{2}{l}{  $ \T[\seq, \varphi \land S_i\varphi]$}\\[1mm]
    (g) & $\T[\seq, B_i \varphi]$ & = & \multicolumn{2}{l}{ $ \T[\f_i(s_{\none} \override{\seq}), \varphi]$ }\\[1mm]
    
    (h) & $\T[\vec{s}, \es_G \alpha]$   & $=$ & \multicolumn{2}{l}{ $\min(\{ \T[\seq, S_i \alpha] \mid i\in G\})$}\\[1mm]
    
    (i) & $\T[\seq, \ek_i \varphi]$ & = & \multicolumn{2}{l}{  $ \T[\seq, \varphi \land \es_i\varphi]$}\\[1mm]
    (j) & $\T[\vec{s}, \ds_G v]$  & $=$ & $\unknown$ & if $v \not\subseteq \seq[n]$\\
        &                      &   & & or $ \forall i \in G,\ i \notin \seq[n]$;\\
	&                      &   & 0 & if $v \notin \bigcup_{i\in G}\oldobservation_i(\seq[n])$;\\
        &                      &   & 1 & otherwise 	\\[1mm]
	(k) & $\T[\vec{s}, \ds_G \varphi]$ & $=$ & $\unknown$ & if $\T[\seq,\varphi] = \T[\seq,\neg\varphi] = \unknown$\\
        &                      &   & & or $ \forall i \in G,\ i \notin \seq[n]$\\
   	&                      &   & 0 & if $\T[\overrightarrow{d\oldobservation}_G(\seq), \varphi] = \unknown$, \\
        &                      &   &   & and $\T[\overrightarrow{d\oldobservation}_G(\seq), \neg \varphi] = \unknown$, \\
	  &                      &   & 1 & otherwise\\[1mm] 

    (l) & $\T[\seq, \dk_i \varphi]$ & = & \multicolumn{2}{l}{  $ \T[\seq, \varphi \land \ds_i\varphi]$}\\[1mm]
    (m) & $\T[\vec{s}, \cs_G v]$  & $=$ & $\unknown$ & if $v \not\subseteq \seq[n]$\\
        &                      &   & & or $ \exists i \in G,\ i \notin \seq[n]$;\\
	&                      &   & 0 & if $v \notin \cc\oldobservation(G,\seq[n])$;\\
        &                      &   & 1 & otherwise \\[1mm]
	
    (n) & $\T[\vec{s}, \cs_G \varphi]$  & $=$ & $\unknown$ & if $\T[s,\varphi] = \T[s,\neg\varphi] = \unknown$\\
        &                      &   & & or $ \exists i \in G,\ i \notin s$;\\
	&                      &   & 0 & if $\T[\overrightarrow{ \cc\oldobservation}(G,\seq), \varphi] = \unknown$,\\
        &                      &   &  & and $\T[\overrightarrow{ \cc\oldobservation}(G,\seq), \neg\varphi] = \unknown$\\
        &                      &   & 1 & otherwise \\[1mm]
    (o) & $\T[\seq, \ck_i \varphi]$ & = & \multicolumn{2}{l}{  $ \T[\seq, \varphi \land \cs_i\varphi]$}\\[1mm]
    (p) & $\T[\vec{s}, \eb_{G} \varphi]$   & $=$ & \multicolumn{2}{l}{  min($\{\T[ \vec{g}, \varphi] \mid \forall \vec{g} \in EF$}\\
        &                                  &     & \multicolumn{2}{l}{where $EF = \ef_{G}(s_{\none} \override{\seq}) \}) $} \\[1mm]
    (q) & $\T[\vec{s}, \db_{G} \varphi]$   & $=$ & \multicolumn{2}{l}{  $\T[ \df_{G}(s_{\none} \override{\seq}), \varphi]$}\\[1mm]
    (r) & $\T[\vec{s}, \cb_{G} \varphi]$   & $=$ & \multicolumn{2}{l}{ min($\{\T[ \vec{g}, \varphi] \mid \forall \vec{g} \in CF \})$,}\\
        &                                  &     & \multicolumn{2}{l}{where $CF = \cf_{G}(\{s_{\none} \override{\seq}\}) $} \\[1mm]

\end{supertabular}

\vspace{2mm} 
\noindent
where: 
$\alpha$ is a variable $v$ or a formula$\varphi$; 
$\seq[n]$ is the final state in sequence $\seq$; 
$\overrightarrow{d\oldobservation}_G(\seq) = [\bigcup_{i\in G}\oldobservation_i(\seq[0]),\dots, \bigcup_{i\in G}\oldobservation_i(\seq[n])]$; 
and, $\overrightarrow{ \cc\oldobservation}(G,\seq)=[\cc\oldobservation(G,\seq[0]),\dots,\cc\oldobservation(G,\seq[n])]$.
\end{definition}

\subsection{Complete Semantics}

Here, we provide the full complete semantics, which includes: 
semantics for the individual modal logic operators using justified perspective functions (item a-g)~\cite{DBLP:conf/aips/Hu0L23};
semantics for the group seeing and knowledge operators using group perspective functions (item h-l)~\cite{DBLP:journals/jair/Hu0L22};
And, the semantics for group belief operators is similar as we provided in Section~3.2.

\begin{definition}[Complete semantics]
\label{def:jpm:complete}
The complete semantics for justified perspective are defined as:

\vspace{2mm}
\noindent
\begin{supertabular}{@{~~}l@{~~}l@{~~}l@{~~}l}
 (a) & $(M,\seq) \vDash r(V_r)$      & iff & $\pi(\seq[n], r(V_r)) = true$\\[1mm]
 (b) & $(M,\seq) \vDash \phi \land \psi$ & iff & $(M,\seq) \vDash \phi$ and $(M,\seq) \vDash \psi$\\[1mm]
 (c) & $(M,\seq) \vDash \neg \varphi$    & iff & $(M,\seq) \not \vDash \varphi$\\[1mm]
 (d) & $(M,\seq) \vDash S_i v$           & iff & $v \in \observation_i(\seq[n])$\\[1mm]
 (e) & $(M,\seq) \vDash S_i \varphi$     & iff & $\forall \vec{g} \in\overrightarrow{\statespacecomplete}^{n+1}$, \\
     &                                   & & $(M,  \vec{g}\override{\observation_i(\seq)} ) \vDash \varphi$ \\
     &                                   &  or & $\forall \vec{g} \in \overrightarrow{\statespacecomplete}^{n+1}$, \\
     &                                   & & $(M,  \vec{g} \override{\observation_i(\seq)} ) \vDash \neg\varphi$\\[1mm]
 (f) & $(M,\seq) \vDash K_i\varphi$      & iff & $(M, \seq ) \vDash \varphi \land S_i \varphi$\\[1mm]
 (g) & $(M,\seq) \vDash B_i \varphi$     & iff & $\forall \vec{g} \in\overrightarrow{\statespacecomplete}^{n+1}$, \\
     &                                   & & $(M,  \vec{g}[\f_i(\seq)]) \vDash \varphi$  \\[1mm]
 (h) & $(M,\seq) \vDash \es_G\alpha$     & iff & $\forall i\in G$, $(M,\vec{s}) \vDash S_i\alpha$\\[1mm]
 (i) & $(M,s) \vDash \ek_G \varphi$      & iff & $(M,s) \vDash (\varphi \land \es_G \varphi)$\\[1mm]
 (j) & $(M,\seq) \vDash \ds_G v$         & iff & $v \in \bigcup_{i\in G}\oldobservation_i(\seq[n])$\\[1mm]
 (k) & $(M,\seq) \vDash \ds_G \varphi$    & iff & $\forall \vec{g} \in\overrightarrow{\statespacecomplete}^{n+1}$, \\
     &                                   & & $(M,\vec{g}\langle \overrightarrow{dO_G}(\seq) \rangle) \vDash \varphi$\\ 
     &                                   & & or $\forall \vec{g} \in\overrightarrow{\statespacecomplete}^{n+1}$, \\
     &                                   & & $(M,\vec{g}\langle \overrightarrow{dO_G}(\seq) \rangle) \vDash \neg\varphi$,\\[1mm]
 (l) & $(M,s) \vDash \dk_G \varphi$   & iff & $(M,s) \vDash (\varphi \land \ds_G \varphi)$\\[1mm]
 (m) & $(M,\seq) \vDash \cs_G v$       & iff & $v \in \cc\oldobservation(G,\seq[n])$\\[1mm]
 (n) & $(M,\seq) \vDash \cs_G \varphi$ & iff & $\forall \vec{g} \in\overrightarrow{\statespacecomplete}^{n+1}$,  \\
    &                                   & & $(M,\vec{g} \langle \overrightarrow{ \cc\oldobservation}(G,\seq) \rangle ) \vDash \varphi$\\
    &                                   & & or $\forall \vec{g} \in\overrightarrow{\statespacecomplete}^{n+1}$, \\
    &                                   & & $(M,\vec{g} \langle \overrightarrow{ \cc\oldobservation}(G,\seq) \rangle ) \vDash \neg\varphi$\\[1mm]
 (o) & $(M,s) \vDash \ck_G \varphi$   & iff & $(M,s) \vDash (\varphi \land \es_G \varphi)$\\[1mm]
 (p) & $(M,\seq) \vDash \eb_{G} \varphi$   & iff & $ \forall \vec{w} \in \ef_G(\vec{s})$, \\
     &                                          &     & $ \forall \vec{g} \in\overrightarrow{\statespacecomplete}^{n+1} , (M, \vec{g}[\vec{w}]) \vDash \varphi$\\[1mm]
 (q) & $(M,\seq) \vDash \db_{G} \varphi$   & iff & $ \forall \vec{g} \in\overrightarrow{\statespacecomplete}^{n+1}, (M, \vec{g}[\df_{G}(\vec{s})]) \vDash \varphi$\\[1mm] 
 (r) & $(M,\seq) \vDash \cb_{G} \varphi$   & iff & $ \forall \vec{w} \in \cf_{G}(\{\vec{s}\})$, \\
     &                                          &     &   $ \forall \vec{g} \in\overrightarrow{\statespacecomplete}^{n+1} , (M, \vec{g}[\vec{w}]) \vDash \varphi$\\[1mm]
\end{supertabular}
\vspace{2mm}

\noindent where: 
$\alpha$ is a variable $v$ or a formula $\varphi$; 
$\seq[n]$ is the final state in sequence $\seq$; 
$ \overrightarrow{\statespacecomplete}^{n+1}$ is the sequence space with length of $n+1$ (same length as $\seq$);
$\vec{g}\langle \seq \rangle = [\vec{g}[0] \langle \seq[0] \rangle ,\dots,\vec{g}[n]\langle \seq[n] \rangle]$;
$\overrightarrow{d\oldobservation}_G(\seq) =$ $  [\bigcup_{i\in G}\oldobservation_i(\seq[0]), \dots, \bigcup_{i\in G}\oldobservation_i(\seq[n])]$; 
and, $\overrightarrow{ \cc\oldobservation}(G,\seq)=[\cc\oldobservation(G,\seq[0]),\dots,\cc\oldobservation(G,\seq[n])]$.
\end{definition}

In the complete semantics, the input of the perspective function does not need to be filled in with the none values to make it a complete-state sequence.
The reason is that even in evaluating seeing relation, the complete semantics overrides the potential partial-state sequence $\vec{O}_i(\seq)$ with the possible complete-state sequence $\vec{g} \in \overrightarrow{\statespacecomplete}^{n+1}$.
That is, all sequences in evaluation are complete-state sequences.

\bibliography{aaai2026}


\end{document}